\documentclass{article}

\usepackage{arxiv}
\usepackage[utf8]{inputenc} 
\usepackage[T1]{fontenc}    
\usepackage{hyperref}       
\usepackage{url}            
\usepackage{booktabs}       
\usepackage{amsfonts}       
\usepackage{nicefrac}       
\usepackage{microtype}      
\usepackage{lipsum}		
\usepackage{graphicx}
\usepackage{natbib}
\usepackage{doi}
\usepackage{wrapfig}
\usepackage{algorithm}
\usepackage[noend]{algpseudocode}
\usepackage[bottom]{footmisc}
\newtheorem{theorem}{Theorem}

\newenvironment{proof}{\paragraph{Proof.}}{\hfill$\square$}

\title{Quantization in Spiking Neural Networks}

\author{ 
{
\hspace{1mm}Bernhard A.~Moser}\thanks{double affiliation: Software Competence Center Hagenberg (SCCH), 4232 Hagenberg, Austria} \\
	Institute of Signal Processing \\
	Johannes Kepler University of Linz\\
	\texttt{bernhard.moser@\{scch.at,jku.at\}} 
	\And
	{\hspace{1mm}Michael Lunglmayr} \\
	Institute of Signal Processing\\
	Johannes Kepler University of Linz, Austria\\
	\texttt{michael.lunglmayr@jku.at} 
	}

\renewcommand{\shorttitle}{\textit{arXiv} Quantization in SNNs}

\hypersetup{
pdftitle={A template for the arxiv style},
pdfsubject={q-bio.NC, q-bio.QM},
pdfauthor={Bernhard A.~Moser, Michael Lunglmayr},
pdfkeywords={Leaky Integrate-and-Fire, Neuromorphic Computing, Spiking Neuron Model, Threshold-based Sampling},
}

\begin{document}
\maketitle

\begin{abstract}
In spiking neural networks (SNN), at each node, an incoming sequence of weighted Dirac pulses is converted into an output sequence of weighted Dirac pulses by a leaky-integrate-and-fire (LIF) neuron model based on spike aggregation and thresholding.
We show that this mapping can be understood as a quantization operator and state a corresponding 
formula for the quantization error by means of the Alexiewicz norm.
This analysis has implications for rethinking reinitialization in the LIF model, leading to the proposal of
{\it reset-to-mod} as a modulo-based reset variant.
\end{abstract}


\section{Introduction}
Though its simplicity, the leaky integrate-and-fire (LIF) neuron model is widely spread in neuromorphic computing and computational neuroscience~\cite{bookGerstner2014,Nunes2022}. In contrast to more biophysically realistic models, such as the Hodgkin-Huxley model, LIF is a middle ground to capture essential features of bio-inspired time-based information processing, and simple enough to be applicable from the point of view of neuromorphic engineering. In particular, in biology spikes show varying shapes and its generation and re-initialization follows a complex dynamics as revealed by the Hodgkin-Huxley differential equations.
In contrast, LIF is based on the following major idealizations. 
(1), it neglects shape information and idealizes spikes as shapeless impulses, (2), it is based on the assumption that the triggering process is realized by a comparison with a given threshold-level, and, (3), the process of triggering and re-initialization acts instantaneously. 

In this paper we take these assumptions as axioms for a mathematical operator that maps a sequence of weighted Dirac impulses to another one.
We ask about properties of this mapping and find that it can be understood as a quantization operation in the space of spike trains.
The standard model of quantization of a single number is integer truncation. 
That is, a given real number $x = n + r$, $n \in \mathbb{Z}$, $r \in (-1,1)$ is mapped to the integer $n$ and the quantization error is given by $r$. 
The resulting mapping $q: \mathbb{R} \rightarrow \mathbb{Z}$ can also be considered from a geometric point of view. 
Consider a grid of vertexes of a tessellation of polytopes induced by the unit ball of the maximum norm, $\|.\|_{\infty}$.
This way, a given point $x$ is element of such a polytope $P$ and mapped to that vertex of $P$ that is closest to the origin.

 
In an abstract sense, we will show that LIF can be understood in the same way.
Instead of the maximum norm, $\|.\|_{\infty}$, the space $\mathbb{S}$ of spike trains is tessellated by means of the unit balls of the Alexiewicz norm 
$\|.\|_{A, \alpha}$, which will be outlined next. Then we get the formula 
$\|\mbox{LIF}_{\alpha, \vartheta}(\eta) - \eta\|_{A, \alpha} < \vartheta$,
where $\mbox{LIF}_{\alpha, \vartheta}$ denotes the LIF mapping based on the leaky parameter $\alpha$, the threshold $\vartheta>0$ and 
the input spike train $\eta \in \mathbb{S}$.

The paper is outlined as follows. In Section~\ref{s:LIF} we reformulate the LIF model, taking different re-initialization variants into account. 
In Section~\ref{s:Alex} we recall the Alexiewicz norm $\|.\|_{A}$, generalize it to leaky variants $\|.\|_{A, \alpha}$ and 
state the main theorem. Section~\ref{s:Evaluations} presents evaluations considering the discussed reinitialization variants, based on Python code available at~\url{https://github.com/LunglmayrMoser/AlexSNN}.

\section{LIF Model and Re-Initialization Variants}
\label{s:LIF}
The flow of information in a LIF model is outlined in Fig.~\ref{fig:Superposition}.
\begin{figure}
	\centering
	\includegraphics[width=0.8\textwidth]{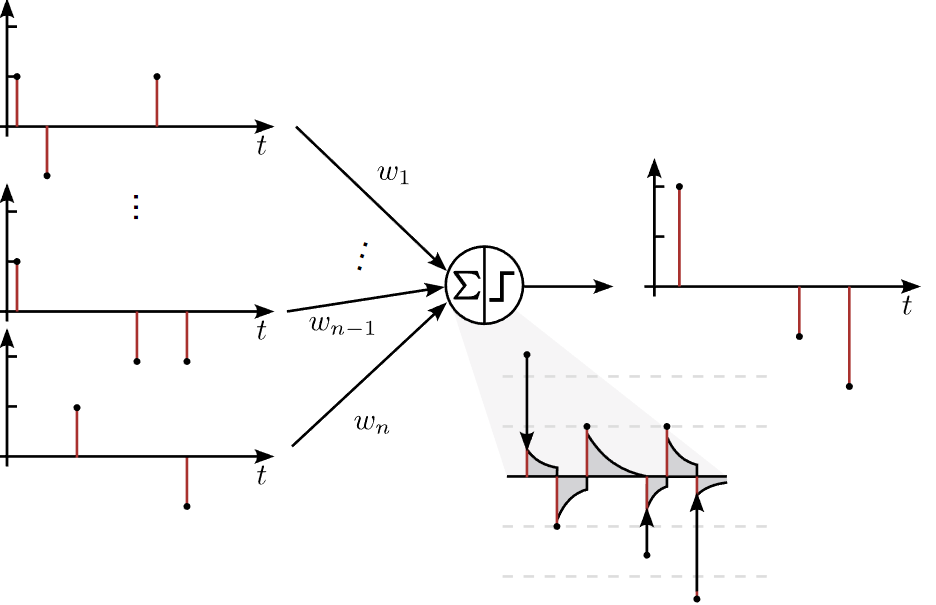}
		\caption{Illustration of steps of information processing by a LIF model. 
		The weighted superposition of incoming spike trains can result in spike amplitudes that can exceed the threshold many times.
				}
	\label{fig:Superposition}
\end{figure}
Arriving from several input channels, spike events trigger the dynamics of 
the membrane potential by a weighted linear superposition across the channels and an integration with a leaky term in time.
The LIF model idealizes the superposition (and, also that of re-initialization) by an instantaneous aggregation. 
Due to this building process, the incoming spike events can equivalently be represented by a single channel of spikes with 
amplitudes resulting from a weighted sum of simultaneous spikes across the input channels.
This way we have to take spike amplitudes of virtually arbitrary magnitude into account. 
In the literature, this aspect is pointed out as a limitation of the LIF model, 
since it deviates from biology in this respect~\cite{bookGerstner2014}. 
As an idealized mathematical model, however, it has its justification, but then 
this effect must also be taken into account in the re-initialization step.
While in the context of LIF there are discussed basically two modes of re-initialization (synonymously, {\it reset}), 
namely {\it reset-to-zero} and {\it reset-by-subtraction}, in Eqn.~(\ref{eq:reset}) we will propose a third one, 
{\it reset-to-mod}, to be mathematically consequent regarding the outlined issue of instantaneous superposition.
According to~\cite{snnTorch2021}, {\it reset-to-zero} means that the potential is reinitialized to zero after firing, while 
{\it reset-by-subtraction} subtracts the $\vartheta$-potential $u_{\vartheta}$ from the membrane's potential that triggers the firing event.
As a third variant we introduce {\it reset-to-mod}, which can be understood as 
instantaneously cascaded application of {\it reset-by-subtraction} according to the factor by which the membrane's potential 
exceeds the threshold which results in a modulo computation. 

For an input spike train $\eta_{\tiny in}(t) = \sum_i a_i \delta_{t_i}(t) $
the mapping $\sum_i b_i \, \delta_{s_i} = \mbox{LIF}_{\vartheta, \alpha}(\eta_{\tiny in})$
is recursively given by 
$
s_{i+1} =
\inf\left\{s\geq s_i
: \,
\left| u_{\vartheta, \alpha}(s_i, s)  \right| 
\geq \vartheta\right\}, 
$
where 
\begin{equation}
\label{eq:u}
u_{\vartheta, \alpha}(t_{\mbox{\tiny event}}, t) := 
 \int^t_{t_{\mbox{\tiny event}}}
	e^{-\alpha (\tau - t_{\mbox{\tiny event}})} 
	\left(		\eta_{\tiny in}(\tau) - \mbox{discharge}(t_{\mbox{\tiny event}}, \tau) \right) d\tau \nonumber
\end{equation}
models the dynamic change of the neuron membrane's potential after an input spike event at time $t_{\mbox{\tiny event}}$.
The process of triggering an output spike is actually a charge-discharge event 
\begin{equation}
\label{eq:reset}
\mbox{discharge}(t_i, \tau) := 
\left\{
\begin{array}[2]{lcl}
	a_i\delta_{t_i}(\tau) & \ldots & \mbox{for {\it reset-to-zero}}, \\
	\mbox{sgn}(a_i)\, \vartheta\, \delta_{t_i}(\tau)  & \ldots & \mbox{for {\it reset-by-subtraction}},  \\
	q(a_i/\vartheta)\, \vartheta\, \delta_{t_i}(\tau) & \ldots & \mbox{for {\it reset-to-mod}}
\end{array}
\right.
\end{equation}
where $\mbox{sgn}(x) \in \{-1,0,1\}$ is the signum and $q(x):=\mbox{sgn}(x)\max\{k \in \mathbb{Z}: k \leq |x|\}$ is the integer truncation quantization function.

\begin{figure}
	\centering
		\includegraphics[width=1\textwidth]{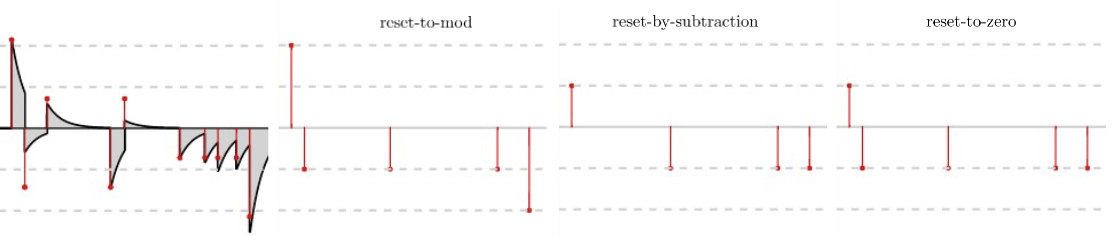}
  		\caption{Example of input spike train with different LIF outputs depending on the re-initialization variant.}
	\label{fig:LIF}
\end{figure}

The {\it reset-by-subtraction} mode can be understood as compensation event so that the net voltage balance of the 
spiking event equals zero, i.e., in case of an output spike with amplitude $\vartheta>0$ the membrane is actually
discharged by this amount. Accordingly, though not always made clear in the literature, see for example~\cite{snnTorch2021},
this assumption has the subtle consequence that an increase of the membrane potential $u$ by multiples $q(u/\vartheta)$ of the threshold level 
$\vartheta$ results in a discharge of the membrane's potential by the same amount, 
that is $q(u/\vartheta)\vartheta$. See Fig.~\ref{fig:LIF} for an illustration and example.

\section{Alexiewicz Norm and Spike Train Quantization}
\label{s:Alex}
For $\eta = \sum_i a_i \delta_{t_i} \in \mathbb{S}$ we introduce the measure
$\|\eta\|_{A, \alpha}:=\max_n\left| \sum_{j=1}^n a_j e^{-\alpha (t_n - t_j)}\right|$, 
which satisfies the axioms of a norm on the vector space $\mathbb{S}$.
For $\alpha = 0$ we obtain the {\it Alexiewicz} norm~\cite{Alexiewicz1948},
which is topologically equivalent to the discrepancy norm~\cite{Moser12UnitBall}. 
The Alexiewicz norm reveals the quantization character of the LIF model.
\begin{theorem}[{\it \bf reset-to-mod} LIF as $\|.\|_{A,\alpha}$-Quantization]
\label{th:quantization}
Given a LIF neuron model with  {\it reset-to-mod}, the threshold $\vartheta>0$, the leaky parameter $\alpha \in [0,\infty]$ 
and the spike train $\eta \in \mathbb{S}$ with amplitudes $a_i \in \mathbb{R}$. 
Then, $\mbox{LIF}_{\vartheta, \alpha}(\eta)$ is a 
$\vartheta$-quantization of $\eta$, i.e., the resulting spike amplitudes 
are multiples of $\vartheta$, where the quantization error is bounded by
\begin{equation}
\label{eq:quantization}
\|\mbox{LIF}_{\vartheta, \alpha}(\eta) - \eta\|_{A, \alpha} < \vartheta. 
\end{equation}
\end{theorem}

\begin{proof}
\begin{wrapfigure}{r}{5.5cm}
  \includegraphics[width=5cm]{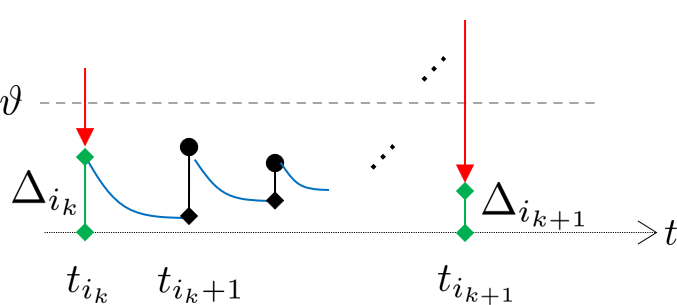}
	\caption{Illustration of Eqn.~(\ref{eq:quantDeltaRecursion}). The red arrows indicate reset by {\it reset-to-mod}.}
 		\label{fig:quantProof}
\end{wrapfigure}

First of all, we introduce the following operation $\oplus$, which is associative and can be handled with like the usual addition if adjacent elements $a_i$ from a spike train $\eta = \sum_i a_i \delta_{t_i}$ are aggregated:
\begin{equation}
\label{eq:pseudoaddition}
a_i \oplus a_{i+1} :=  e^{-\alpha (t_{i+1}- t_{i})} a_i  +  a_{i+1}.
\end{equation}
This way we get a simpler notation when aggregating convolutions, e.g., 
\[
a_i \oplus \ldots \oplus a_j = \sum_{k=i}^j e^{-\alpha (t_{j}- t_{k})} a_k.
\] 

For the discrete version we re-define 
$a_{i_k} \oplus a_{i_{k+1}} :=  \beta^{(i_{k+1}- i_{k})} a_{i_k}  +  a_{i_{k+1}}$, if $i_{k}$ and $i_{k+1}$ refer to adjacent spikes at time $i_k$, resp. $i_{k+1}$. Further, we denote $q[z]:= \mbox{sgn}(z)\, \left\lfloor |z|\right\rfloor$ which is the ordinary quantization due to integer truncation, e.g. $q[1.8] = 1$, $q[-1.8]=-1$, where $\left\lfloor |z| \right\rfloor = \max\{n\in \mathbb{N}_0:\, n \leq |z|\}$.  

After fixing notation let us consider a spike train $\eta = \sum_j a_i \delta_{t_{i}}$. 
Without loss of generality we may assume that $\vartheta = 1$.
We have to show that $\|\mbox{LIF}_{1, \alpha}(\eta)  - \eta\|_{A, \alpha} < 1$, which is equivalent 
to the discrete condition that $\forall n: \max_n \left|\sum_{i=1}^n \hat{a}_i \right| < 1$,
where  $\eta - \mbox{LIF}_{1, \alpha}(\eta) = \sum_i \hat{a}_i \delta_{t_i}$.
Set $\hat{s}_k := \hat{a}_0 \oplus \cdots \oplus \hat{a}_k$. We have to show that
$\max_{k}|\hat{s}_k| < 1$. The proof is based on induction and leads the problem back to the standard quantization by truncation.

Suppose that at time $t_{i_{k-1}}$ after re-initialization by {\it reset-to-mod} we get the residuum $\Delta_{i_{k-1}}$ as membrane potential that is the starting point for the integration after $t_{i_{k-1}}$. 
Note that 
\begin{equation}
\mbox{LIF}_{1, \alpha}(\eta)|_{t = t_k} = q(\Delta_{i_{k-1}} \oplus a_{i_{k-1}+1} \cdots \oplus a_{i_{k}}) \nonumber
\end{equation}
Then, as illustrated in Fig.~\ref{fig:quantProof} the residuum $\Delta_{i_{k}}$ at the next triggering event $t_{i_{k}}$ is obtained by the equation
\begin{equation}
\label{eq:quantDeltaRecursion}
\Delta_{i_{k}} = \Delta_{i_{k-1}} \oplus a_{i_{k-1}+1} \oplus \ldots \oplus  a_{i_k} - q[\Delta_{i_{k-1}} \oplus \ldots \oplus  a_{i_k}].
\end{equation}
Note that due to the thresholding condition of LIF we have
\begin{equation}
\label{eq:thcond}
|\Delta_{i_{k}} \oplus a_{i_{k}+1} \oplus \ldots \oplus  a_j| < 1
\end{equation}
for $j \in \{i_{k}+1, \ldots, i_{k+1}-1\}$.
For the $\oplus$-sums $\hat{s}_{i_k}$ we have
\begin{equation}
\label{eq:ahat}
\hat{s}_{i_{k+1}} = \hat{s}_{i_{k}} \oplus a_{i_{k}+1} \cdots a_{i_{k+1}-1} \oplus 
\left( 
a_{i_{k+1}} - q[\Delta_{i_{k}} \oplus a_{i_{k}+1} \oplus \ldots \oplus a_{i_{k+1}}]
\right).
\end{equation}

Note that $\hat{s}_0 = \Delta_{i_0}= a_0 - q[a_0]$, then for induction we assume that up to index $k$ to have
\begin{equation}
\label{eq:quantInduction}
\hat{s}_{i_k} = \Delta_{i_k}.
\end{equation}

Now, using (\ref{eq:quantInduction}), Equation~(\ref{eq:ahat}) gives
\begin{eqnarray}
\label{eq:induction}
\hat{s}_{i_{k+1}} & = & \Delta_{i_k} \oplus a_{i_{k}+1} \oplus \ldots \oplus 
a_{i_{k+1}-1} \oplus 
\left( 
a_{i_{k+1}} - q[\Delta_{i_{k}} \oplus a_{i_{k+1}} \oplus \ldots \oplus a_{i_{k+1}}]
\right) \nonumber \\
& = & \Delta_{i_k} \oplus a_{i_{k}+1} \oplus \ldots \oplus 
a_{i_{k+1}-1} \oplus 
a_{i_{k+1}} - q[\Delta_{i_{k}} \oplus a_{i_{k+1}} \oplus \ldots \oplus a_{i_{k+1}}], \nonumber\\
 & = & \Delta_{i_{k+1}}
\end{eqnarray}
proving (\ref{eq:quantInduction}), which together with (\ref{eq:thcond}) ends the proof showing that 
$|\hat{s}_k|<1$ for all $k$.
\end{proof}

\section{Evaluations}
\label{s:Evaluations}
We evaluate the distribution of the quantization error 
(\ref{eq:quantization}) by means of box whisker diagrams 
for different re-initialization variants depending on the number of spikes in the spike train and different
distributions of amplitudes. Fig.~\ref{fig:QuantizationError1} shows the result for incoming spike amplitudes that are below threshold.
In this case, (\ref{eq:quantization}) holds for {\it reset-to-mod} and {\it reset-by-subtraction}. As expected, both variants behave similarly. 
For {\it reset-to-zero} the bound (\ref{eq:quantization})  only holds approximately for larger leaky parameter $\alpha$, see first row of Fig.~\ref{fig:QuantizationError1}.
For {\it reset-to-mod} and {\it reset-by-subtraction}, With increasing number of spikes one can observe a concentration of measure effect, see, e.g.,~\cite{Vershynin2018}). 

\begin{figure}[ht]
	\centering
		\includegraphics[width=0.325\textwidth]{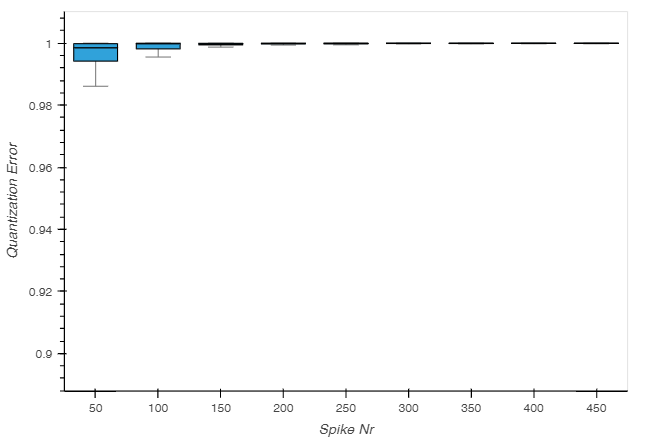}
	\includegraphics[width=0.325\textwidth]{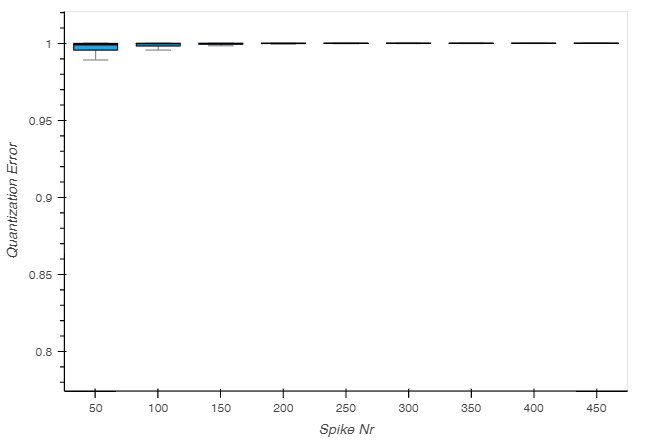}
	\includegraphics[width=0.325\textwidth]{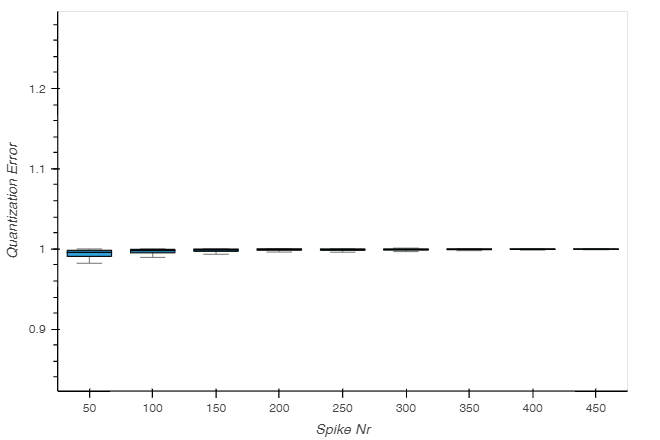}
	\includegraphics[width=0.325\textwidth]{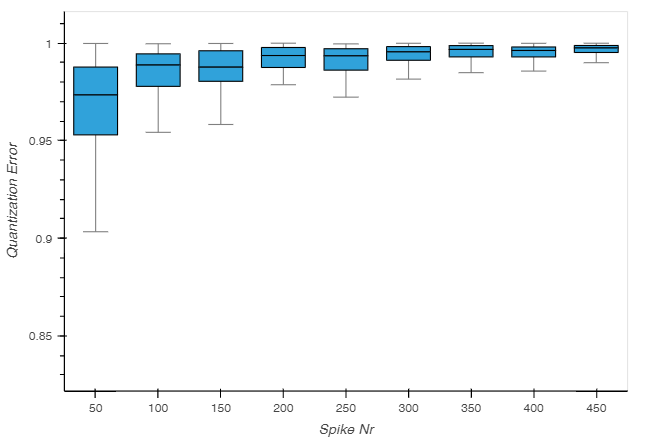}
	\includegraphics[width=0.325\textwidth]{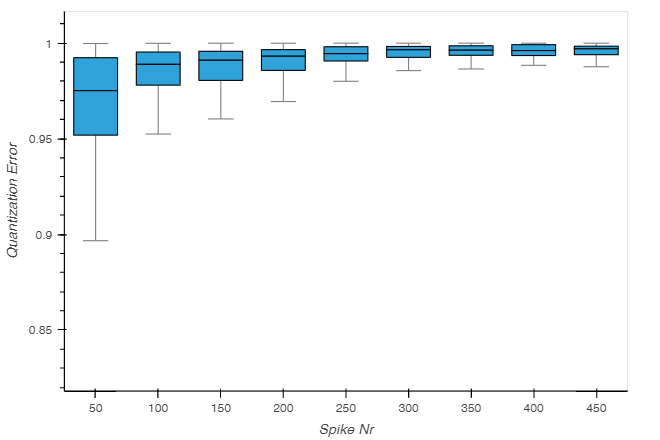}
	\includegraphics[width=0.325\textwidth]{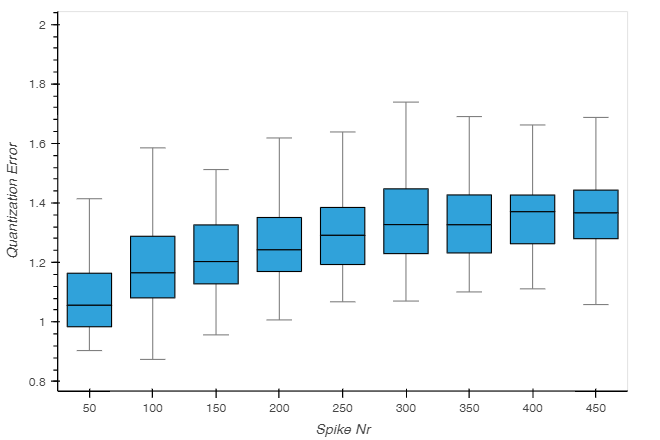}
		\caption{
		Evaluation of (\ref{eq:quantization}) for {\it reset-to-mod}, {\it reset-by-subtraction} and
		{\it reset-to-zero} (1.st/2nd/3rd column), based on spike trains with spike amplitudes in $[-\vartheta, \vartheta]$ and $100$ runs.
		The 1st row refers to $\alpha = 1$ and the 2nd row to $\alpha = 0.1$.		
		}
			\label{fig:QuantizationError1}
\end{figure}
However, as shown in Fig.~\ref{fig:QuantizationError2}, if the incoming spike amplitudes are not below threshold anymore Eqn.~(\ref{eq:quantization}) only holds for {\it reset-to-mod}. In this case the measure of concentration effect become more apparent for smaller leaky parameter. For the other variants the quantization error increases in average with the number of spikes, and the theoretical bound~(\ref{eq:quantization}), proven for {\it reset-to-mod}, does not hold anymore.
\begin{figure}[ht]
	\centering
	\includegraphics[width=0.325\textwidth]{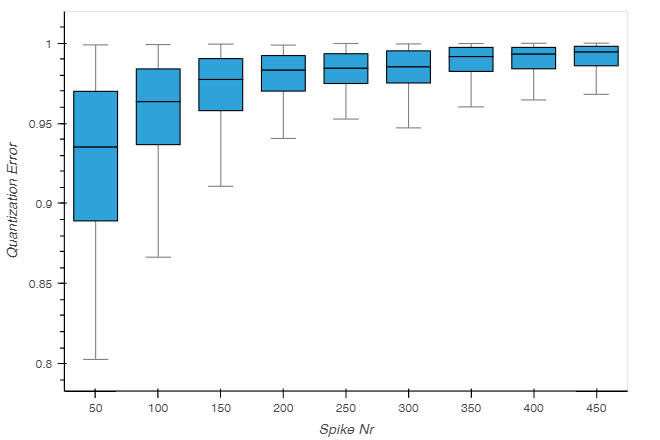}
	\includegraphics[width=0.325\textwidth]{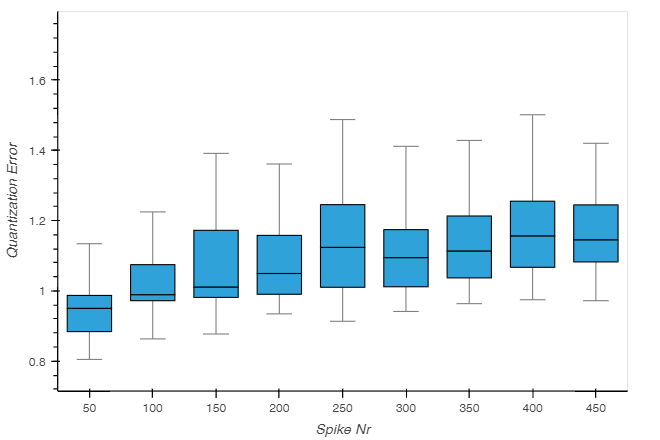}
	\includegraphics[width=0.325\textwidth]{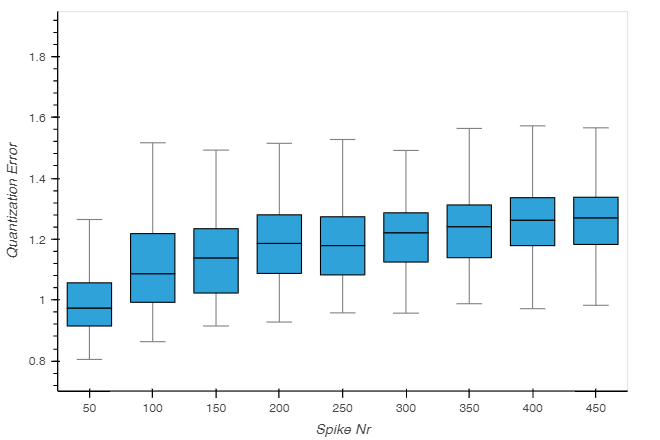}
	\includegraphics[width=0.325\textwidth]{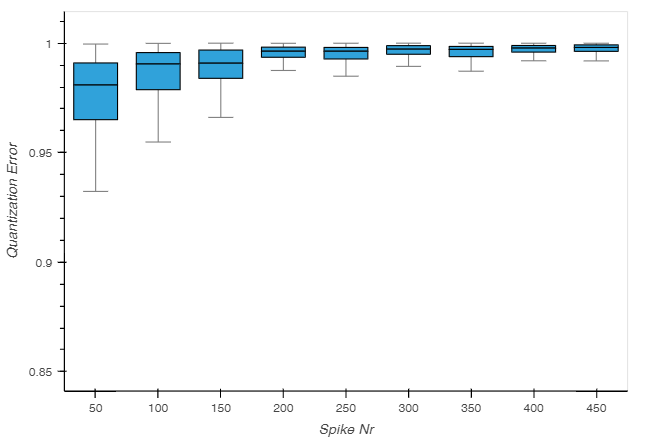}
	\includegraphics[width=0.325\textwidth]{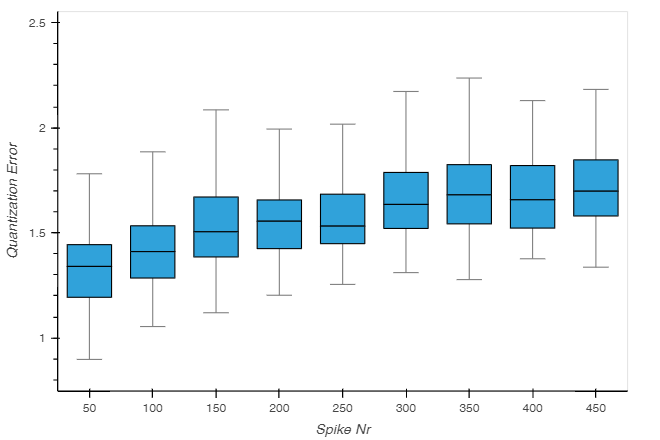}
	\includegraphics[width=0.325\textwidth]{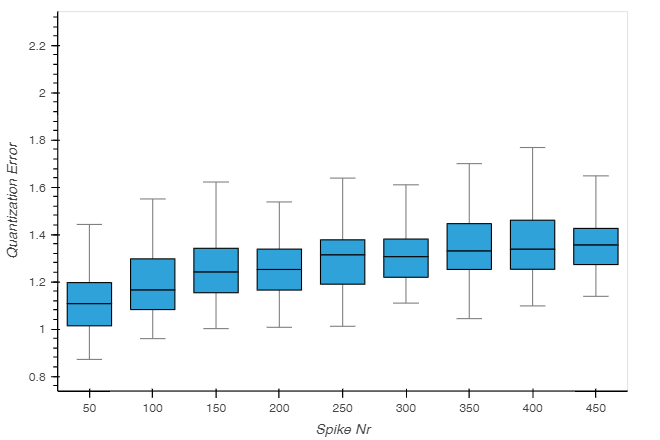}
		\caption{The same as in Fig.~\ref{fig:QuantizationError1} but with spike amplitudes in $[-3/2 \vartheta, 3/2 \vartheta]$.}
			\label{fig:QuantizationError2}
\end{figure}

\section{Conclusion}
\label{s:Conclusion}
In this paper we provide a novel view on the leaky integrate-and-fire model as quantization operator in the Alexiewicz norm.
This analysis gives rise to rethinking the re-initialization modes {\it reset-to-zero} and {\it reset-by-subtraction} that are commonly 
used in the context of spiking neural networks. These re-initialization modes only hold under restricted conditions while 
our proposed variant {\it reset-to-mod} satisfies the derived quantization bound under general conditions.
This general quantization error formula leads to new 
error bounds for LIF and SNNs, such as a quasi-isometry relation in analogy to threshold-based 
sampling~\cite{Moser2017Similarity,MoserLunglmayr2019QuasiIsometry}. 
Examples can be found in the github repository~\url{https://github.com/LunglmayrMoser/AlexSNN} and~\cite{moser2023arXiv_SNNAlexTop}. 


\section*{Acknowledgements}
This work was supported (1) by the 'University SAL Labs' initiative of Silicon Austria Labs (SAL) and its Austrian partner universities for applied fundamental research for electronic based systems, (2) by Austrian ministries BMK, BMDW, and the State of Upper-Austria in the frame of SCCH, part of the COMET Programme managed by FFG, and (3) by the {\it NeuroSoC} project funded under the Horizon Europe Grant Agreement number 101070634.

\bibliographystyle{unsrtnat}
\bibliography{references}
\end{document}